\documentclass{article} 
\usepackage{iclr2024_conference,times}

\usepackage{amsmath,amsfonts,bm}

\def\eqref#1{equation~\ref{#1}}

\def\1{\bm{1}}

\DeclareMathAlphabet{\mathsfit}{\encodingdefault}{\sfdefault}{m}{sl}
\SetMathAlphabet{\mathsfit}{bold}{\encodingdefault}{\sfdefault}{bx}{n}

\usepackage{hyperref}
\usepackage{url}

\usepackage{array}

\usepackage[frozencache,cachedir=.]{minted}
\setminted{mathescape}
\newminted[pycode]{py}{}
\newminted[pycodesmall]{py}{fontsize=\small}
\newminted[pymathcode]{py}{escapeinside=||}
\newminted[pymathnumberedcode]{py}{escapeinside=||,linenos}

\usepackage[varqu]{inconsolata}
\usepackage[T1]{fontenc}

\usepackage{j}

\usepackage{tablefootnote}

\usepackage{arydshln} 
\usepackage{multirow}

\theoremstyle{definition}
\newtheorem{definition}{Definition}[section]
\newtheorem{example}{Example}[section]
\theoremstyle{plain}
\newtheorem{lemma}{Lemma}[section]
\newtheorem{theorem}{Theorem}[section]

\newcommand{\G}{\mathcal{G}}
\newcommand{\Sn}{\mathcal{S}_n}
\newcommand{\F}{\mathcal{F}}

\newcommand{\ttt}{\texttt}
\newcommand{\tbf}{\textbf}
\newcommand{\isoclass}{\tilde}
\newcommand{\canonized}{\overline}

\DeclareMathOperator{\Orb}{Orb}
\DeclareMathOperator{\Stab}{Stab}
\DeclareMathOperator{\Aut}{Aut}

\usepackage{tikz}
\usepackage{pgfmath}

\tikzset{vertexOrdered/.style={outer sep=-1.5}}
\tikzset{vertexUnordered/.style={shape=circle,fill,inner sep=0,outer sep=2}}

\newlength{\graphScale}
\newcommand{\vertexSize}{\footnotesize}

\newcommand{\boxedtikz}[2][]{\vcenter{\hbox{\tikz[#1]{#2}}}}

\newcommand{\graphThree}[1]{
\boxedtikz[x=\graphScale,y=\graphScale]{
  \node (0) at (0.0, 0)              [vertexOrdered] {\vertexSize{0}};
  \node (1) at (1.0, 0)              [vertexOrdered] {\vertexSize{1}};
  \pgfmathparse{-sqrt(3) / 2}
  \node (2) at (0.5, \pgfmathresult) [vertexOrdered] {\vertexSize{2}};
  #1
}}

\newcommand{\graphThreeUnordered}[1]{
\boxedtikz[x=\graphScale,y=\graphScale]{
  \node (0) at (0.0, 0)              [vertexUnordered] {};
  \node (1) at (1.0, 0)              [vertexUnordered] {};
  \pgfmathparse{-sqrt(3) / 2}
  \node (2) at (0.5, \pgfmathresult) [vertexUnordered] {};
  #1
}}

\newcommand{\graphFour}[1]{
\boxedtikz[x=\graphScale,y=\graphScale]{
  \node (0) at (0,  0) [vertexOrdered] {\vertexSize{0}};
  \node (1) at (1,  0) [vertexOrdered] {\vertexSize{1}};
  \node (2) at (1, -1) [vertexOrdered] {\vertexSize{2}};
  \node (3) at (0, -1) [vertexOrdered] {\vertexSize{3}};
  #1
}}

\newcommand{\graphFourUnordered}[1]{
\boxedtikz[x=\graphScale,y=\graphScale]{
  \node (0) at (0,  0) [vertexUnordered] {};
  \node (1) at (1,  0) [vertexUnordered] {};
  \node (2) at (1, -1) [vertexUnordered] {};
  \node (3) at (0, -1) [vertexUnordered] {};
  #1
}}

\newcommand{\atomSize}{\scriptsize}
\newlength{\moleculeScale}
\newcommand{\Hy}{\atomSize{H}}
\newcommand{\Ca}{\atomSize{C}}
\newcommand{\Ox}{\atomSize{O}}
\newcommand{\Ni}{\atomSize{N}}
\newcommand{\Bo}{\atomSize{B}}

\tikzset{atom/.style={outer sep=-2.5}}

\newcommand{\NO}{
\(
\boxedtikz[x=\moleculeScale,y=\moleculeScale]{
  \node (Ni) at (0, 0) [atom] {\Ni};
  \node (Ox) at (.8, 0) [atom] {\Ox};

  \draw (Ni) edge[line width=2.1pt] (Ox);
  \draw (Ni) edge[line width=1.3pt,white,line cap=round] (Ox);
}
\)
}

\newcommand{\HtwoO}{
\(
\boxedtikz[x=0.8\moleculeScale,y=0.8\moleculeScale]{
  \node (Ha) at (0,    0   ) [atom] {\Hy};
  \node (Ox) at (0.79, 0.61) [atom] {\Ox};
  \node (Hb) at (1.58, 0   ) [atom] {\Hy};

  \draw (Ha) -- (Ox);
  \draw (Hb) -- (Ox);
}
\)
}

\newcommand{\HtwoOtwo}{
\(
\boxedtikz[x=0.7\moleculeScale,y=0.7\moleculeScale]{
  \node (Oa) at ( 0.00,  0.00) [atom] {\Ox};
  \node (Ob) at ( 1.00,  0.00) [atom] {\Ox};
  \node (Ha) at (-0.87, -0.50) [atom] {\Hy};
  \node (Hb) at ( 1.87, -0.50) [atom] {\Hy};

  \draw (Ha) -- (Oa);
  \draw (Oa) -- (Ob);
  \draw (Ob) -- (Hb);
}
\)
}

\newcommand{\Ethylene}{
\(
\boxedtikz[x=0.7\moleculeScale,y=0.7\moleculeScale]{
  \node (Cw)  at ( 0.00,  0.00) [atom] {\Ca};
  \node (Ce)  at ( 1.00,  0.00) [atom] {\Ca};
  \node (Hnw) at (-0.87,  0.50) [atom] {\Hy};
  \node (Hne) at ( 1.87,  0.50) [atom] {\Hy};
  \node (Hsw) at (-0.87, -0.50) [atom] {\Hy};
  \node (Hse) at ( 1.87, -0.50) [atom] {\Hy};

  \draw (Hnw) -- (Cw);
  \draw (Hsw) -- (Cw);
  \draw (Hne) -- (Ce);
  \draw (Hse) -- (Ce);

  \draw (Cw) edge[line width=2.1pt] (Ce);
  \draw (Cw) edge[line width=1.3pt,white,line cap=round] (Ce);
}
\)
}

\newcommand{\BoricAcid}{
\(
\boxedtikz[x=0.55\moleculeScale,y=0.55\moleculeScale]{
  \node (B)  at ( 0.00,  0.00) [atom] {\Bo};

  \node (Onw)  at (-0.50,  0.87) [atom] {\Ox};
  \node (Osw)  at (-0.50, -0.87) [atom] {\Ox};
  \node (Oe)   at ( 1.00,  0.00) [atom] {\Ox};

  \node (Hnw)  at (-0.50 - 0.9,  0.87) [atom] {\Hy};
  \node (Hsw)  at (-0.50 + 0.9 * 0.5, -0.87 - 0.9 * 0.87) [atom] {\Hy};
  \node (He)   at ( 1.00 + 0.9 * 0.5,  0.00 + 0.9 * 0.87) [atom] {\Hy};

  \draw (B) -- (Onw);
  \draw (B) -- (Osw);
  \draw (B) -- (Oe);

  \draw (Onw) -- (Hnw);
  \draw (Osw) -- (Hsw);
  \draw (Oe) -- (He);
}
\)
}

\usepackage{xcolor}

\definecolor{lightlightgray}{gray}{0.8}

\title{Entropy Coding of \\Unordered Data Structures}

\author{Julius Kunze\\
  University College London\\
  \ttt{juliuskunze@gmail.com}
  \And
  Daniel Severo\\
  University of Toronto and Vector Institute\\
  \ttt{d.severo@mail.utoronto.ca}
  \And
  Giulio Zani\\
  University of Amsterdam\\
  \ttt{g.zani@uva.nl}
  \And
  Jan-Willem van de Meent\\
  University of Amsterdam\\
  \ttt{j.w.vandemeent@uva.nl}
  \And
  James Townsend\\
  University of Amsterdam\\
  \ttt{j.h.n.townsend@uva.nl}
}

\iclrfinalcopy
\begin{document}
\maketitle
\begin{abstract}
  We present shuffle coding, a general method for optimal compression of
  sequences of unordered objects using bits-back coding. Data structures that
  can be compressed using shuffle coding include multisets, graphs,
  hypergraphs, and others. We release an implementation that can easily be
  adapted to different data types and statistical models, and demonstrate that
  our implementation achieves state-of-the-art compression rates on a range of
  graph datasets including molecular data.
\end{abstract}

\section{Introduction}\label{sec:introduction}

The information stored and communicated by computer hardware, in the form of
strings of bits and bytes, is inherently ordered. A string has a first and last
element, and may be indexed by numbers in \(\mathbb{N}\), a totally ordered
set. For data like text, audio, or video, this ordering carries meaning.
However, there are also numerous data structures in which the `elements' have no
meaningful order. Common examples include graphs, sets and multisets, and
`map-like' datatypes such as JSON. Recent applications of machine learning to
molecular data benefit from large datasets of molecules, which are graphs with
vertex and edge labels representing atom and bond types (some examples are
shown in \Cref{tab:molecule-examples} below).  All of these data are
necessarily stored in an ordered manner on a computer, but the order then
represents \emph{redundant information}. This work concerns optimal lossless
compression of unordered data, and we seek to eliminate this redundancy.

\begin{table}[ht]
\centering
\caption{
  Examples of molecules and their order information. The `discount' column
  shows the saving achieved by shuffle coding by removing order information
  (see eq.\ \ref{eq:rate}). For each molecule \(\mathbf{m}\), \(n\) is the
  number of atoms and \(\abs{\Aut(\mathbf{m})}\) is the size of the
  automorphism group.  All values are in bits, and $\log$ denotes the binary
  logarithm.
}\label{tab:molecule-examples}
\begin{tabular}[t]{@{}r@{\;\,}lccc@{}}
  \toprule
  \multicolumn{2}{c}{\multirow{2}{*}{\textbf{Molecular
  structure}}}&\textbf{Permutation}&\textbf{Symmetry}&\textbf{Discount}\\
              &            &\(\log n!\)&\(\log\abs{\Aut(\mathbf{m})}\) & \(\log n! -
              \log\abs{\Aut(\mathbf{m})}\) \\
  \midrule
  Nitric oxide     &\NO        &\(1.00\)  &\(0.00\) &\(1.00\) \\\addlinespace
  Water            &\HtwoO     &\(2.58\)  &\(1.00\) &\(1.58\) \\\addlinespace
  Hydrogen peroxide&\HtwoOtwo  &\(4.58\)  &\(1.00\) &\(3.58\) \\\addlinespace
  Ethylene         &\Ethylene  &\(9.49\)  &\(3.00\) &\(6.49\) \\\addlinespace
  Boric acid       &\BoricAcid &\(12.30\) &\(2.58\) &\(9.71\) \\
  \bottomrule
\end{tabular}
\end{table}

Recent work by \citet{severo2023} showed how to construct an optimal lossless
codec for (unordered) multisets from a codec for (ordered) vectors, by storing
information in an ordering. Their method depends on the simple structure of
multisets' automorphism groups, and does not extend to other unordered objects
such as unlabeled graphs. In this paper we overcome this issue and develop
\emph{shuffle coding}, a method for constructing codecs for general `unordered
objects' from codecs for `ordered objects'.
Our definitions of ordered and unordered objects are based on the concept of
`combinatorial species' \citep{joyal1981, bergeron1997}, originally developed to assist
with the enumeration of combinatorial structures.
They include multisets,
as well as all of the other unordered data structures mentioned above,
and many more.

Although the method is applicable to any unordered object, we focus
our experiments on unordered (usually referred to as `unlabeled')
graphs, as these are a widely used data type, and the improvements
in compression rate from removing order information are large (as
summarized in \Cref{tab:results-summary}).
We show that shuffle coding can achieve significant improvements relative
to existing methods, when compressing
unordered graphs under the Erdős-Rényi \(G(n, p)\) model of \citet{erdHos1960evolution}
as well as the recently proposed Pólya's urn-based model
of \citet{severo2023rec}.
Shuffle coding extends to graphs with vertex and edge attributes, such as
the molecular and social network datasets of TUDatasets \citep{Morris+2020}, which
are compressed in \Cref{sec:experiments}.
We release source code\footnote{Source code, data and results are available at \url{https://github.com/juliuskunze/shuffle-coding}.} with
straightforward interfaces to enable future applications of shuffle coding with
more sophisticated models and to classes of unordered objects other than graphs.

\section{Background}
The definitions for ordered and unordered objects are given in \Cref{sec:comb-classes}.
Entropy coding is reviewed in \Cref{sec:codecs}.
Examples are given throughout the section for clarification.

\subsection{Permutable classes}\label{sec:comb-classes}
For \(n\in \mathbb{N}\), we let \([n]\coloneqq \{0, 1, \ldots, n-1\}\), with
\([0] = \emptyset\). The symmetric group of permutations on \([n]\), i.e.\
bijections from \([n]\) to \([n]\), will be denoted by \(\Sn\).  Permutations
compose on the left, like functions, i.e., for \(s, t\in \Sn\), the product
\(st\) denotes the permutation formed by performing \(t\) then \(s\).

Permutations are represented as follows
\begin{equation}
  \graphThree{
    \draw [->] (0) -- (2);
    \draw [->] (1) -- (0);
    \draw [->] (2) -- (1);
  }
  = (2, 0, 1)\in \mathcal{S}_3.
\end{equation}
The glyph on the left-hand side represents the permutation that maps \(0\) to
\(2\), \(1\) to \(0\) and \(2\) to \(1\). This permutation can also be
represented concretely by the vector \((2, 0, 1)\).

Concepts from group theory, including subgroups, cosets, actions, orbits, and stabilizers are
used throughout. We provide a brief introduction in \Cref{app:OS-theorem}.

We will be compressing objects which can be `re-ordered' by applying
permutations. This is formalized in the following definition:

\begin{definition}[Permutable class\footnote{
    This definition is very close to that of a `combinatorial species', the
    main difference being that we fix a specific \(n\). See discussion in
    \citet[66--67]{yorgey2014}.
  }]\label{def:perm-class}
  For \(n\in \mathbb{N}\), a \emph{permutable class} of order \(n\) is a set
  \(\F\), equipped with a left group action of the permutation group \(\Sn\) on \(\F\), which we
  denote with the \(\cdot\) binary operator.  We refer to elements of \(\F\) as
  \emph{ordered objects}.
\end{definition}


\begin{example}[Length \(n\) strings]\label{example:strings}
  For a fixed set \(X\), let \(\F_n = X^n\), that is, length \(n\) strings of
  elements of \(X\), and let \(\Sn\) act on a string in \(\F_n\) by rearranging
  its elements.

  Taking \(X\) to be the set of ASCII characters, we can define a permutable
  class of ASCII strings, with action by rearrangement, for example
  \begin{equation}
  \graphFour{
      \draw [->] (0) -- (2);
      \draw [->] (1) -- (0);
      \draw [->] (2) -- (1);
      \draw [->] (3) edge[loop left] (3);
    }
  \;
  \cdot\,\ttt{"Team"} = \ttt{"eaTm"}.
  \end{equation}
\end{example}

\newcommand{\simpleGraph}{\graphFour{
    \draw (0) -- (1); \draw (0) -- (2);
    \draw (1) -- (2); \draw (2) -- (3);
}}
\begin{example}[Simple graphs \(\G_n\)]\label{example:graphs}
  Let \(\G_n\) be the set of simple graphs with vertex set \([n]\).
  Specifically, an element \(g\in\G_n\) is a set of `edges', which are
  unordered pairs of elements of \([n]\). We define the action of \(\Sn\) on a
  graph by moving the endpoints of each edge in the direction of the arrows,
  for example
  \begin{equation}
\graphFour{
    \draw [->] (0) -- (2);
    \draw [->] (1) -- (0);
    \draw [->] (2) -- (1);
    \draw [->] (3) edge[loop left] (3);
  }\;\cdot
  \simpleGraph
  =\graphFour{
    \draw (0) -- (1);
    \draw (0) -- (2);
    \draw (1) -- (2);
    \draw [ultra thick,white] (1) -- (3);
    \draw (1) -- (3);
  }.
  \end{equation}
\end{example}

Our main contribution in this paper is a general method for compressing
\emph{unordered} objects. These may be defined formally in terms of the
equivalence classes, known as orbits, which comprise objects that are identical
up to re-ordering (see \Cref{app:OS-theorem} for background):


\begin{definition}[Isomorphism, unordered objects]\label{def:isomorphism}
  For two objects \(f\) and \(g\) in a permutable class \(\F\), we say that
  \(f\) is \emph{isomorphic} to \(g\), and write \(f\simeq g\), if there
  exists \(s\in \Sn\) such that \(g = s\cdot f\) (i.e.\ if \(f\) and \(g\) are
  in the same orbit under the action of \(\Sn\)). Note that the relation
  \(\simeq\) is an equivalence relation.  For \(f\in\F\) we use
  \(\isoclass{f}\) to denote the equivalence class containing \(f\), and
  \(\widetilde{\F}\) to denote the quotient set of equivalence classes. We
  refer to elements \(\isoclass{f}\in\widetilde{\F}\) as \emph{unordered
  objects}.
\end{definition}

For the case of strings in \cref{example:strings}, an object's isomorphism
class is characterized by the multiset of elements contained in the string.
For the simple graphs in \cref{example:graphs}, the generalized isomorphism in
\cref{def:isomorphism} reduces to the usual notion of graph isomorphism.
We can define a shorthand notation for unordered graphs, with points at the
nodes instead of numbers:
\begin{equation}
\graphFourUnordered{
    \draw (0) -- (1);
    \draw (0) -- (2);
    \draw (1) -- (2);
    \draw (2) -- (3);
  }\;
  \coloneqq
  \;\widetilde{\simpleGraph}
  \;.
\end{equation}

Using this notation, the unordered simple graphs on three vertices, for
example, can be written:
\begin{equation}
  \widetilde{\G_3} = \left\{\;
    \graphThreeUnordered{
    }\;,\quad
    \graphThreeUnordered{
      \draw (0) -- (1);
    }\;,\quad
    \graphThreeUnordered{
      \draw (0) -- (1);
      \draw (0) -- (2);
    }\;,\quad
    \graphThreeUnordered{
      \draw (0) -- (1);
      \draw (0) -- (2);
      \draw (1) -- (2);
    }
  \;\right\}.
\end{equation}

Finally, we define the subgroup of \(\Sn\) which contains the symmetries of a
given object \(f\):

\begin{definition}[Automorphism group]
  For an element \(f\) of a permutable class \(\F\), we let \(\Aut(f)\)
  denote the \emph{automorphism group} of \(f\), defined by
  \begin{equation}
    \Aut(f)\coloneqq \{s\in\Sn \mid s\cdot f = f\}.
  \end{equation}
  This is the stabilizer subgroup of \(f\) under the action of
  \(\Sn\).
\end{definition}

The elements of the automorphism group of the
simple graph from \cref{example:graphs} are:
\begin{equation}
  \Aut\left(\simpleGraph\right) = \left\{
    \graphFour{
      \draw [->] (0) edge[loop left] (0);
      \draw [->] (1) edge[loop right] (1);
      \draw [->] (2) edge[loop right] (2);
      \draw [->] (3) edge[loop left] (3);
    }\;,\quad
    \graphFour{
      \draw [->] (0) edge[bend left] (1);
      \draw [->] (1) edge[bend left] (0);
      \draw [->] (2) edge[loop right] (2);
      \draw [->] (3) edge[loop left] (3);
    }
  \right\}.
\end{equation}

\subsubsection{Canonical orderings}\label{sec:canonical-orderings}
To define a codec for unordered objects, we will introduce the notion of a
`canonical' representative of each equivalence
class in \(\widetilde{\F}\). This allows us, for example, to check whether two
ordered objects are isomorphic, by mapping both to the canonical representative
and comparing.

\begin{definition}[Canonical ordering]\label{def:canon-order}
  A \emph{canonical ordering} is an operator
  \(\canonized{\cdot}:\F\rightarrow\F\), such that
  \begin{enumerate}
    \item For \(f\in \F\), we have \(\canonized{f}\simeq f\).
    \item For \(f,g\in \F\), \(\canonized{f}=\canonized{g}\) if and only if
      \(f\simeq g\).
  \end{enumerate}
\end{definition}


For strings, any sorting function satisfies properties 1 and 2 and is therefore
a valid canonical ordering. For graphs, the canonical orderings we use are
computed using the \ttt{nauty} and \ttt{Traces} libraries \citep{mckay2014}.
The libraries provide a function, which we call \ttt{canon\_perm}, which, given
a graph \(g\), returns a permutation \(s\) such that \(s\cdot g =
\canonized{g}\). In addition to \ttt{canon\_perm}, \ttt{nauty} and \ttt{Traces} can compute the automorphism group of a given
graph, via a function which we refer to as \ttt{aut}.\footnote{
In fact, a list of generators for the group is computed, rather than the entire
group, which may be very large.
} 

Our method critically depends on the availability of such a function for a given permutation class. While permutable objects other than graphs cannot be directly canonized by
\ttt{nauty} and \ttt{Traces}, it is often possible to embed objects into
graphs in such a way that the structure is preserved and the canonization
remains valid (see \citet{anders2021}). We use an embedding of edge-colored graphs into vertex-colored graphs in order to canonize and compress graphs with edge attributes (which are not directly supported by nauty/traces). We leave more systematic approaches
to canonizing objects from permutable classes as an interesting direction
for future work\footnote{
  \citet{schweitzer2019} describe generic methods for canonization starting
  from a constructive definition of permutable objects using `hereditarily
  finite sets' (i.e.\ not using the species definition).
}.

\subsection{Codecs}\label{sec:codecs}

We fix a set \(M\) of prefix-free binary messages, and a length function
\(l\colon M\rightarrow [0, \infty)\), which measures the number of physical
bits required to represent values in \(M\). Our method requires stack-like (LIFO) codecs, such as those based on the range
variant of asymmetric numeral systems (rANS), to save bits corresponding to the redundant order using bits-back \citep{townsend2019}.

\begin{definition}[(Stack-like) codec]\label{def:codec}
  A \emph{stack-like codec} (or simply \emph{codec}) for a set \(X\) is an invertible function
  \begin{equation}
    \ttt{encode} : M\times X \rightarrow M.
  \end{equation}
  We call a codec \emph{optimal} for a probability distribution over \(X\) with mass function \(P\) if for any \(m\in M\) and \(x\in X\),
  \begin{equation}\label{eq:codec-rate}
    l(\ttt{encode}(m, x)) \approx l(m) + \log\frac{1}{P(x)}.\footnote{
This condition, with a suitable definition of \(\approx\), is equivalent to
rate-optimality in the usual Shannon sense, see \citet{townsend2020a}.}
  \end{equation}
  We refer to \(\log\frac{1}{P(x)}\) as the \emph{optimal rate} and to the inverse of $\ttt{encode}$ as $\ttt{decode}$. Since $\ttt{decode}$ has to be implemented in practice, we treat it as an explicit part of a codec below.
\end{definition}

The \ttt{encode}
function requires a pre-existing message as its first input. Therefore, at the
beginning of encoding we set \(m\) equal to some fixed, short initial message
\(m_0\), with length less than 64 bits. As in other entropy coding methods, which invariably have some small constant overhead, this 'initial bit cost' is amortized as we compress more data. 

We will assume access to three primitive codecs provided by rANS. These are
\begin{itemize}
  \item \ttt{Uniform(n)}, optimal for a uniform distribution on \(\{\ttt{0},
    \ttt{1}, \ldots, \ttt{n-1}\}\).
  \item \ttt{Bernoulli(p)}, optimal for a Bernoulli distribution with
    probability \ttt{p}.
  \item \ttt{Categorical(ps)}, optimal for a categorical distribution with
    probability vector \ttt{ps}. 
\end{itemize}

These primitive codecs can be composed to implement codecs for strings and simple graphs. In \cref{app:ordered-codecs}, we show such a string codec optimal for a
distribution where each character is drawn i.i.d.\ from a categorical with
known probabilities, and a codec for simple graphs optimal for the Erdős-Rényi
\(G(n, p)\) model, where each edge's existence is decided by an independent
draw from a Bernoulli with known probability parameter. We will use these
codecs for ordered objects as a component of shuffle coding.

There is an implementation-dependent limit on the parameter \ttt{n} of
\ttt{Uniform} and on the number of categories for \ttt{Categorical}. In the
64-bit rANS implementation which we wrote for our experiments, this limit is
\(2^{48}\).  This is not large enough to, for example, cover \(\Sn\) for large
\(n\), and therefore permutations must be encoded and decoded sequentially, see
\Cref{app:perm-codecs}. For details on the implementation of the primitive rANS
codecs listed above, see \citet{duda2009,townsend2021}.

\section{Codecs for unordered objects}\label{sec:method}
Our main contribution in this paper is a generic codec for unordered
objects, i.e.\ a codec respecting a given probability
distribution on \(\widetilde{\F}\). We first derive an expression for the optimal rate
that this codec should achieve, then in
\Cref{sec:unordered-codec-implementation} we describe the codec itself.


To help simplify the presentation, we will use the following generalization of
exchangeability from sequences of random variables to arbitrary permutable
classes:

\begin{definition}[Exchangeability]\label{def:exchangeability}
  For a probability distribution \(P\) defined on a permutable class \(\F\),
  we say that \(P\) is \emph{exchangeable} if isomorphic objects have equal
  probability under \(P\), i.e.\ if
  \begin{equation}\label{eq:exchangeability}
    f\simeq g\Rightarrow P(f) = P(g).
  \end{equation}
\end{definition}


We can assume, without loss of modeling power, that unordered objects are
generated by first generating an ordered object \(f\) from an exchangeable
distribution and then `forgetting' the order by projecting \(f\) onto its
isomorphism class \(\isoclass{f}\):
\begin{lemma}[Symmetrization]
  For any distribution \(Q\) on a class of unordered objects
  \(\widetilde{\F}\), there exists a unique exchangeable distribution \(P\) on
  ordered objects \(\F\) for which
  \begin{equation}\label{eq:induced-distr}
    Q(\isoclass{f}) = \sum_{g\in\isoclass{f}} P(g).
  \end{equation}
\end{lemma}
\begin{proof}
  For existence, set \(P(f)\coloneqq
  Q(\isoclass{f})/\abs{\isoclass{f}}\) for \(f\in\F\), and note that
  \(g\in\isoclass{f}\Rightarrow\isoclass{g} = \isoclass{f}\).
  For uniqueness, note that \cref{def:exchangeability} implies
  that the restriction of \(P\) to any particular class must be uniform, which
  completely determines \(P\).
\end{proof}

We will model real-world permutable objects using an exchangeable model,
which will play the role of \(P\) in \cref{eq:induced-distr}. To further
simplify our rate expression we will also need the following application of the
orbit-stabilizer theorem (see \Cref{app:OS-theorem} for more detail), which is
visualized in \Cref{fig:OS-visualization}:
\begin{lemma}\label{lemma:OS-aut}
  Given a permutable class \(\F\), for each object \(f\in\F\), there is a
  fixed bijection between the left cosets of \(\Aut(\canonized{f})\) in \(\Sn\)
  and the isomorphism class \(\isoclass{f}\). This is induced by the function
  \(\theta_f : \Sn\rightarrow \F\) defined by \(\theta_f(s) \coloneqq s\cdot
  \canonized{f}\).  This implies that
  \begin{equation}\label{eq:OS-permutable-object}
    \abs{\isoclass{f}} = \frac{\abs{\Sn}}{\abs{\Aut(f)}} =
    \frac{n!}{\abs{\Aut(f)}}.
  \end{equation}
\end{lemma}

\begin{proof}
  Follows directly from the orbit-stabilizer theorem (\cref{theorem:OS}) and
  the definitions of \(\Aut\), \(\canonized{f}\) and \(\isoclass{f}\).
\end{proof}

\begin{figure}[t]
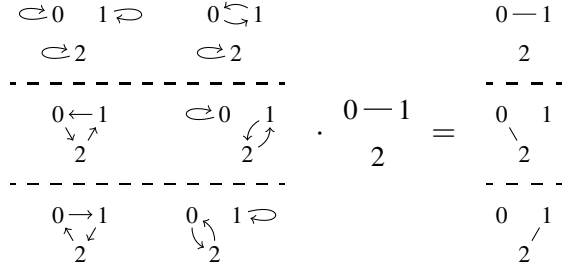

\centering
\(
\begin{array}{@{}cc@{}}
\graphThree{
  \draw [->] (0) edge[loop left] (0);
  \draw [->] (1) edge[loop right] (1);
  \draw [->] (2) edge[loop left] (2);
}
&
\graphThree{
  \draw [->] (0) edge[bend right] (1);
  \draw [->] (1) edge[bend right] (0);
  \draw [->] (2) edge[loop left] (2);
}
\\\addlinespace\hdashline\addlinespace
\graphThree{
  \draw [->] (1) -- (0);
  \draw [->] (2) -- (1);
  \draw [->] (0) -- (2);
}
&
\graphThree{
  \draw [->] (1) edge[bend right] (2);
  \draw [->] (2) edge[bend right] (1);
  \draw [->] (0) edge[loop left] (0);
}
\\\addlinespace\hdashline\addlinespace
\graphThree{
  \draw [->] (0) -- (1);
  \draw [->] (1) -- (2);
  \draw [->] (2) -- (0);
}
&
\graphThree{
  \draw [->] (0) edge[bend right] (2);
  \draw [->] (2) edge[bend right] (0);
  \draw [->] (1) edge[loop right] (1);
}
\end{array}
\quad\,
\scalebox{1.2}{\(
\cdot\;
\graphThree{\draw (0) -- (1);}
=\)}
\quad\,
\begin{array}{@{}c@{}}
\graphThree{\draw (0) -- (1);}\\\addlinespace\hdashline\addlinespace
\graphThree{\draw (0) -- (2);}\\\addlinespace\hdashline\addlinespace
\graphThree{\draw (1) -- (2);}
\end{array}
\)
\caption{
  Visualization of \cref{lemma:OS-aut}. For a fixed graph \(g\),
  the six elements \(s\in\mathcal{S}_3\) can be partitioned according to the
  value of \(s\cdot g\). The three sets in the partition are the left cosets of
  \(\Aut(g)\).
}\label{fig:OS-visualization}
\end{figure}

For any $f \in \F$, this allows us to express the right hand side of
\cref{eq:induced-distr} as:
\begin{equation}\label{eq:nice-Q-expression}
  \sum_{g\in\isoclass{f}} P(g)
    = \abs{\isoclass{f}} P(f)
    = \frac{n!}{\abs{\Aut(f)}}P(f)
\end{equation}
where the first equality follows from exchangeability of \(P\), and the second
from \cref{eq:OS-permutable-object}.  Finally, from
\cref{eq:induced-distr,eq:nice-Q-expression}, we can immediately write down the
following optimal rate expression, which a codec on unordered objects should achieve:
\begin{equation}\label{eq:rate}
  \log\frac{1}{Q(\isoclass{f})} =
  \underbracket[0.15ex]{\log\frac{1}{P(f)}}_{\text{Ordered rate}} -
  \underbracket[0.15ex]{\log\frac{n!}{\abs{\Aut(f)}}}_\text{Discount}.
\end{equation}

Note that only the \(\log 1 / P(f)\) term depends on the choice of model. The
\(\log (n!/\abs{\Aut(f)})\) term can be computed directly from the data, and is
the `discount' that we get for compressing an \emph{unordered} object vs.\
compressing an ordered one. The discount is larger for objects which have a
smaller automorphism group, i.e.\ objects which \emph{lack symmetry}. It can be
shown that almost all simple graphs have a trivial automorphism group for large
enough \(n\), see e.g.\ \citet[Chapter 9]{bollobás_2001}, and thus in practice
the discount is usually equal to or close to \(\log n!\).


\subsection{
  Achieving the target rate for unordered objects
}\label{sec:unordered-codec-implementation}
How can we achieve the optimal rate in \cref{eq:rate}? In \cref{app:ordered-codecs}
we give examples of codecs for ordered strings and simple graphs which achieve
the `ordered rate'. To operationalize the negative `discount' term, we can use
the `bits-back with ANS' method introduced by \citet{townsend2019}, the key
idea being to \emph{decode} an ordering as part of an \emph{encode} function
(see line 3 in the code below).

The value of the negative term in the rate provides a hint at how exactly to
decode an ordering: the discount is equal to the logarithm of the number of
cosets of \(\Aut(\canonized{f})\) in \(\Sn\), so a uniform codec for those
cosets will consume exactly that many bits. \Cref{lemma:OS-aut} tells us that
there is a direct correspondence between the cosets of \(\Aut(\canonized{f})\)
and the set \(\isoclass{f}\), so if we uniformly decode a choice of coset, we
can reversibly map that to an ordering of \(f\).

The following is an implementation of shuffle coding, showing, on the right,
the effect of the steps on message length.
\begin{pymathnumberedcode}
def encode(m, f):                                    |$\text{Effect on message length:}$|
  f_canon = action_apply(canon_perm(f), f)
  m, s = UniformLCoset(f_canon.aut).decode(m)        |$-\log\frac{n!}{\abs{\Aut(f)}}$|
  g = action_apply(s, f_canon)
  m = P.encode(m, g)                                 |$+\log\frac{1}{P(f)}$|
  return m

def decode(m):
  m, g = P.decode(m)
  s_ = inv_canon_perm(g)
  f_canon = action_unapply(s_, f)
  m = UniformLCoset(f_canon.aut).encode(m, s_)
  return m, f_canon
\end{pymathnumberedcode}
The \ttt{encode} function accepts a pair \ttt{(m, f)}, and reversibly
\emph{decodes} a random choice \ttt{g} from the isomorphism class of \ttt{f}.
This is done using a uniform codec for left cosets, \ttt{UniformLCoset}, which
we discuss in detail in \Cref{app:perm-codecs}. The canonization on line 2 is
necessary so that the decoder can recover the chosen coset and encode it on
line 12. While the codec technically maps between $M \times \widetilde{\F}$ and $M$, we avoid representing equivalence classes explicitly as sets, and
instead use a single element of the class as a representative. Thus the encoder
accepts any \ttt{f} in the isomorphism class being encoded, and the decoder
then returns the canonization of \ttt{f}. Similarly,
\ttt{UniformLCoset.encode} accepts any element of the coset, and
\ttt{UniformLCoset.decode} returns a canonical coset element.

\subsection{Initial bits}\label{sec:initial-bits}
The increase in message length from shuffle coding is equal to the optimal rate
in \cref{eq:rate}. However, the decode step on line 3 of the
\ttt{encode} function assumes that there is already some information in the
message which can be decoded. At the very beginning of encoding, these `initial
bits' can be generated at random, but they are unavoidably encoded into the
message, meaning that for the first object, the discount is not realized.
This constant initialization overhead means that the rate, when compressing
only one or a few objects, is not optimal, but tends to the optimal rate if
more objects are compressed, as the overhead is amortized.

\section{Related work}\label{sec:related}
To date, there has been a significant amount of work on compression of what we
refer to as `ordered' graphs, see \citet{besta2019} for a comprehensive survey.
Compression of `unordered' graphs, and unordered objects in general, has been
less well studied, despite the significant potential benefits of removing order
information (see \Cref{tab:results-summary}).  The work of \citet{varshney2006}
is the earliest we are aware of to discuss the theoretical bounds for
compression of sets and multisets, which are unordered strings.

\citet{choi2012} discuss the optimal rate for unordered graphs (a special
case of our eq.\ \ref{eq:rate}), and present a compression method called
`structural ZIP' (SZIP), which asymptotically achieves the rate
\begin{equation}
  \log\frac{1}{P_\mathrm{ER}(g)} - n\log n + O(n),
\end{equation}
where \(P_\mathrm{ER}\) is the Erdős-Rényi \(G(n, p)\) model.  Compared to our
method, SZIP is less flexible in the sense that it only applies to simple
graphs (without vertex or edge attributes), and it is not an entropy coding
method, thus the model \(P_\mathrm{ER}\) cannot be changed easily. On the other
hand, SZIP can achieve good rates on single graphs, whereas, because of the
initial bits issue (see \Cref{sec:initial-bits}), our method only achieves the
optimal rate on \emph{sequences} of objects. We discuss this issue further and
provide a quantitative comparison in \Cref{sec:experiments}.

\citet{steinruecken2014b,steinruecken2015,steinruecken2016} provides a range of
specialized methods for compression of various ordered and unordered
permutable objects, including multisets, permutations, combinations and
compositions.
Steinruecken's approach is similar to ours in that explicit probabilistic
modeling is used, although different methods are devised for each kind of
object rather than attempting a unifying treatment as we have done.

Our method can be viewed as a generalization of the framework for multiset
compression presented in \citet{severo2023}, which also used `bits-back with
ANS' \citep[BB-ANS;][]{townsend2019,townsend2021}.
\citet{severo2023} use interleaving to reduce the
initial bits overhead and achieve an optimal rate when compressing a
\emph{single} multiset (which can also be applied to a sequence of multisets), whereas the method presented in this paper is optimal
only for sequences of unordered objects (including sequences of multisets).
However, as mentioned in \Cref{sec:introduction}, their method only works for
multisets and not for more general unordered objects.  

There are a number of recent works on deep generative modeling of graphs (see
\citet{zhu2022} for a survey), which could be applied to entropy coding to
improve compression rates. Particularly relevant is \citet{chen2021}, who
optimize an evidence lower-bound (ELBO), equivalent to an upper-bound on the
rate in \cref{eq:rate}, when \(P\) is
not exchangeable.  Finally, the `Partition and Code'
\citep[PnC;][]{bouritsas2021} method uses neural networks to compress unordered
graphs. We compare to PnC empirically in \Cref{tab:PnC-comparison}.  PnC is
also specialized to graphs, although it does employ probabilistic modeling to
some extent.


\section{Experiments}\label{sec:experiments}

To demonstrate the method experimentally, we first applied it to the TUDatasets
graphs \citep{Morris+2020}, with a very simple Erdős-Rényi \(G(n, p)\) model
for \(P\).  \Cref{tab:results-summary} shows a summary, highlighting the
significance of the discount achieved by shuffle coding. We compressed a
dataset at a time (note that for each high-level graph type there are multiple
datasets in TUDatasets).

To handle graphs with discrete vertex and edge attributes, we treated all
attributes as independent and identically distributed (i.i.d.) within each
dataset. For each dataset, the codec computes and encodes a separate empirical
probability vector for vertices and edges, as well as an empirical \(p\)
parameter, and the size \(n\) of each graph. We use run-length encoding for these meta-data, described in detail in
\Cref{sec:param-coding}. Some datasets in TUDatasets contain graphs with
continuous attributes. We did not encode these attributes, since for these
values lossy compression would usually be more appropriate, and the focus of
this work is on lossless.

\begin{table}[ht]
  \centering
  \caption{
    For the TUDatasets, this table shows the significance of the discount term -
    in \cref{eq:rate}. With an Erdős-Rényi (ER) model, with edge probability
    adapted to each dataset, the percentage improvement (Discount) is the
    difference between treating the graph as ordered (Ordered ER) and using
    Shuffle coding to forget the order (Shuffle coding ER). Rates are measured
    in bits per edge.
  }\label{tab:results-summary}
  \begin{tabular}{@{}lccc@{}}
    \toprule
    Graph type     &Ordered ER&Shuffle coding ER&Discount\\
    \midrule
    Small molecules&2.11      &1.14             &46\%    \\
    Bioinformatics &9.20      &6.50             &29\%    \\
    Computer vision&6.63      &4.49             &32\%    \\
    Social networks\tablefootnote{
      Three of the 24 social network datasets, REDDIT-BINARY, REDDIT-MULTI-5K,
      REDDIT-MULTI-12K, were excluded because compression running time was too
      long.
    }              &3.98      &2.97             &26\%    \\
    Synthetic      &5.66      &2.99             &47\%    \\
    \bottomrule
  \end{tabular}
\end{table}

We also compared directly to \citet{bouritsas2021}, who used a more
sophisticated neural method to compress graphs (upper part of
\Cref{tab:PnC-comparison}). They reported results for six of the datasets from
the TUDatasets with vertex and edge attributes removed, and for two of the six
they reported results which included vertex and edge attributes. Because
PnC requires training, it was evaluated on a random test subset of each
dataset, whereas shuffle coding was evaluated on entire datasets.

We found
that for some types of graphs, such as the bioinformatics and social network
graphs, performance was significantly improved by using a Pólya urn (PU)
preferential attachment model for ordered graphs introduced by \citet{severo2023rec}.
In this model, a sequence of edges is sampled, where the probability of an edge being connected to a specific node is approximately proportional to the number of edges already connected to that node. Such a `rich-get-richer' dynamic is plausibly present in the formation of many real-world graphs, explaining the urn model's good performance.
It treats edges as a set, and we
were able to use an inner shuffle codec for sets to encode the edges,
demonstrating the straightforward compositionality of shuffle coding. See
\Cref{app:pu} for details. The average initial bit cost per TU dataset in \Cref{tab:PnC-comparison} is $0.01$ bits per edge for both ER and PU, demonstrating good amortization.

As mentioned in \Cref{sec:related}, SZIP achieves a good rate for single
graphs, whereas shuffle coding is only optimal for sequences of graphs.
In the lower part of \Cref{tab:PnC-comparison}, we compare the `net
rate', which is the increase in message length from shuffle coding the graphs,
assuming some existing data is already encoded into the message.  The fact that
shuffle coding `just works' with any statistical model for ordered
graphs is a major advantage of the method, as demonstrated by the fact that
we were easily able to improve on the Erdős-Rényi results by swapping in a
recently proposed model.

\newcolumntype{C}{>{\centering}p{3.5em}}
\begin{table}[ht]
  \caption{
    Comparison between shuffle coding, with Erdős-Rényi (ER) and our Pólya urn (PU)
    models, and the best results obtained by PnC \citep{bouritsas2021} and SZIP
    \citep{choi2012} for each dataset. We also show the discount realized by shuffle coding. Each SZIP comparison is on a
    single graph, and thus for shuffle coding we report the optimal (\emph{net})
    compression rate, that is the additional cost of compressing that graph
    assuming there is already some compressed data to append to. All
    measurements are in bits per edge.
  }\label{tab:PnC-comparison}
  \centering
  \begin{tabular}{@{}llCCCr@{}l@{}}
    \toprule
                                &&&\multicolumn{2}{c}{Shuffle coding}\\ \cmidrule(lr){4-5}
    &Dataset                     &Discount&ER        &PU         &\multicolumn{2}{c}{PnC}\\\midrule
    \tbf{Small molecules} & MUTAG                       &2.77&\tbf{1.88}&2.66       &2.45&\(\pm\)0.02       \\
    &MUTAG (with attributes)     &2.70&\tbf{4.20}&4.97       &4.45&                  \\
    &PTC\_MR                     &2.90&\tbf{2.00}&2.53       &2.97&\(\pm\)0.14       \\
    &PTC\_MR (with attributes)   &2.87&\tbf{4.88}&5.40       &6.49&\(\pm\)0.54       \\
    &ZINC\_full                  &3.11&\tbf{1.82}&2.63       &1.99&                  \\
    \tbf{Bioinformatics}&PROTEINS                    &2.48&3.68      &\tbf{3.50} &3.51&\(\pm\)0.23 \\
    \tbf{Social networks}&IMDB-BINARY                 &0.97&2.06      &1.50       &\tbf{0.54}&            \\
    &IMDB-MULTI                  &0.88&1.52      &1.14       &\tbf{0.38}&
    \\\addlinespace\midrule\addlinespace
    &                            &&\multicolumn{2}{c}{Shuffle coding (net)}\\ \cmidrule(lr){4-5}
    &Dataset                     &Discount&ER        &PU &\multicolumn{2}{c}{SZIP}\\\midrule
    \tbf{SZIP}&Airports (USAir97)          &1.12&5.09      &\tbf{2.90} &\multicolumn{2}{c}{3.81}\\
    &Protein interaction (YeastS)&3.55&6.84      &\tbf{5.70} &\multicolumn{2}{c}{7.05}\\
    &Collaboration (geom)        &3.45&8.30      &\tbf{4.41} &\multicolumn{2}{c}{5.28}\\
    &Collaboration (Erdos)       &7.80&7.00      &\tbf{4.37} &\multicolumn{2}{c}{5.08}\\
    &Genetic interaction (homo)  &3.97&8.22      &\tbf{6.77} &\multicolumn{2}{c}{8.49}\\
    &Internet (as)               &7.34&8.37      &\tbf{4.47} &\multicolumn{2}{c}{5.75}\\
    \bottomrule
  \end{tabular}
\end{table}

We report speeds in \Cref{app:speed}. Our implementation has not yet been optimized. 
One thing
that will not be easy to speed up is canonical ordering, since for this we use
the \ttt{nauty} and \ttt{Traces} libraries, which have already been heavily
optimized. Fortunately, those calls are currently only 10 percent of the
overall time, and we believe there is significant scope for optimization of the
rest.


\section{Limitations and future work}
\tbf{Time complexity.} Shuffle coding relies on computing an object's canonical ordering and automorphism group, for which no polynomial-time algorithm is known for graphs.
In consequence, while \ttt{nauty} and \ttt{Traces} solve this problem efficiently for various graph classes, it is impractical in the worst case. This limitation can be overcome by approximating an object's canonical ordering, instead of calculating it exactly. This introduces a trade-off between speed and compression rate in
the method, and lowers runtime complexity to polynomial time. We leave
a detailed description of that more sophisticated method to future work.

\tbf{Initial bits.} A limitation of the version of shuffle coding presented in this paper is
that it only achieves an optimal rate for sequences; the rate `discount'
cannot be realized in the one-shot case, as explained in \Cref{sec:initial-bits}. However, it is possible to overcome this by
interleaving encoding and decoding steps, as done in the `bit-swap' method of \citet{kingma2019}. Information can be eagerly encoded during the progressive decoding of the coset, reducing the initial bits needed by shuffle coding from \(O(\log
n!)\) to \(O(\log n)\). This is a generalization of the multiset coding method described by \citet{severo2023}. We again defer
a detailed description to future work. 

\tbf{Models.} Unlike PnC, we do not rely on compute-intensive learning or hyperparameter tuning. Shuffle coding achieves state-of-the-art compression rates when using simple models with minimal parameters. There is currently active research on deep generative models for graphs,
see \citet{zhu2022} for a survey. We expect improved rates for shuffle coding when combined with such neural models.

\section{Conclusion}
A significant proportion of the data which needs to be communicated and stored
is fundamentally unordered. We have presented shuffle coding, the first general
method which achieves an optimal rate when compressing sequences of unordered
objects. We have also implemented experiments which demonstrate the practical
effectiveness of shuffle coding for compressing many kinds of graphs, including
molecules and social network data. We look forward to future work applying the
method to other forms of unordered data, and applying more sophisticated
probabilistic generative models to gain improvements in compression rate.

\subsubsection*{Acknowledgments}
James Townsend acknowledges funding under the project VI.Veni.212.106, financed
by the Dutch Research Council (NWO). We thank Ashish Khisti for discussions
and encouragement, and Heiko Zimmermann for feedback on the paper.

\pagebreak

\printbibliography
\pagebreak
\appendix
\section{Group actions, orbits and stabilizers}\label{app:OS-theorem}
This appendix gives the definitions of group actions, orbits and stabilizers as
well as a statement and proof of the orbit-stabilizer theorem, which we make
use of in \cref{sec:method}.  We use the shorthand \(H\le G\) to mean that
\(H\) is a subgroup of \(G\), and for \(g\in G\), we use the usual notation,
\(gH\coloneqq\{gh \mid h\in H\}\) and \(Hg\coloneqq\{hg \mid h\in H\}\) for
left and right cosets, respectively.

\begin{definition}[Group action]
  For a set \(X\) and a group \(G\), a \emph{group action}, or simply
  \emph{action}, is a binary operator
  \begin{equation}
    \cdot_G : G\times X \rightarrow X
  \end{equation}
  which respects the structure of \(G\) in the following sense:
  \begin{enumerate}
    \item The identity element \(e\in G\) is neutral, that is \(e\cdot_G x =
      x\).
    \item The operator \(\cdot_G\) respects composition. That is, for \(g, h\in
      G\),
      \begin{equation}g\cdot_G(h\cdot_G x) = (gh)\cdot_G x.\end{equation}
  \end{enumerate}
  We will often drop the subscript \(G\) and use infix \(\cdot\) alone where
  the action is clear from the context.
\end{definition}

\begin{definition}[Orbit]
  An action of a group \(G\) on a set \(X\) induces an equivalence relation
  \(\sim_G\) on \(X\), defined by
  \begin{equation}
    x\sim_G y\quad\text{if and only if there exists}\quad g\in G\quad\text{such
    that}\quad y = g\cdot x.
  \end{equation}
  We refer to the equivalence classes induced by \(\sim_G\) as \emph{orbits},
  and use \(\Orb_G(x)\) to denote the orbit containing an element \(x\in
  X\). We use \(X/G\) to denote the set of orbits, so for each \(x\in X\),
  \(\Orb_G(x)\in X/G\).
\end{definition}

\begin{definition}[Stabilizer subgroup]
  For an action of a group \(G\) on a set \(X\), for each \(x\in X\), the
  \emph{stabilizer}
  \begin{equation}
    \Stab_G(x) \coloneqq \{g\in G \mid g\cdot x = x\}
  \end{equation}
  forms a subgroup of \(G\).
\end{definition}

We make use of the orbit-stabilizer theorem in \Cref{sec:method}. Here we give
a statement and brief proof of this well-known theorem.

\begin{theorem}[Orbit-stabilizer theorem]\label{theorem:OS}
  For an action of a finite group \(G\) on a set \(X\), for each \(x\in
  X\), the function \(\theta_x \colon G\rightarrow X\) defined by
  \begin{equation}
    \theta_x(g) \coloneqq g\cdot x
  \end{equation}
  induces a bijection from the left cosets of \(\Stab_G(x)\) to \(\Orb_G(x)\).
  This implies that the orbit \(\Orb_G(x)\) is finite and
  \begin{equation}
    \abs{\Orb_G(x)} = \frac{\abs{G}}{\abs{\Stab_G(x)}}.
  \end{equation}
\end{theorem}
\begin{proof}
  We show that \(\theta_f\) induces a well defined function on the left-cosets
  of \(\Stab_G(x)\), which we call \(\tilde{\theta}_f\). Specifically, we
  define
  \begin{equation}
    \tilde{\theta}_f(g\Stab_G(x)) \coloneqq g \cdot x,
  \end{equation}
  and show that \(\tilde{\theta}_f\) is injective and surjective.

  To see that \(\tilde{\theta}_f\) is well defined and injective, note that
  \begin{align}
    h\in g\Stab_G(x)
      &\iff g^{-1}h\in \Stab_G(x)\\
      &\iff g^{-1}h\cdot x = x\\
      &\iff g\cdot x = h\cdot x,
  \end{align}
  using the definition of \(\Stab_G\).

  For surjectivity, we have
  \begin{align}
    y\in \Orb_G(x)
      &\implies \exists g\in G\;\mathrm{s.t.}\; y = g\cdot x\\
      &\implies y = \tilde{\theta}_f(g\Stab_G(x))
  \end{align}
  using the definition of \(\Orb_G\).
\end{proof}

In \Cref{sec:Autcodec}, it will be helpful to have an explicit bijection
between \(G\) and the Cartesian product \(\Orb_G(x)\times\Stab_G(x)\). This
requires a way of selecting a canonical element from each left coset of
\(\Stab_G(x)\) in \(G\). This is similar to the canonical ordering of
\cref{def:canon-order}:

\begin{definition}[Transversal]\label{def:transversal}
  For a group \(G\) with subgroup \(H\le G\), a \emph{transversal} of the
  left cosets of \(H\) in \(G\) is a mapping \(t:G\rightarrow G\) such that
  \begin{enumerate}
    \item For all \(g\in G\), we have \(t(g) \in gH\).
    \item For all \(f, g\in G\), if \(f\in gH\), then \(t(f) = t(g)\).
  \end{enumerate}
\end{definition}

Given such a transversal, we can setup the bijection mentioned above:
\begin{lemma}\label{lemma:OS-bijection}
  Let \(G\) be a group acting on a set \(X\). If, for \(x\in X\), we have a
  transversal \(t_x\) of the left cosets of \(\Stab_G(x)\) in \(G\), then we
  can form an explicit bijection between \(G\) and
  \(\Orb_G(x)\times\Stab_G(x)\).
\end{lemma}
\begin{proof}
  For \(g\in G\), let
  \begin{align}
    o_x(g) &\coloneqq g\cdot x       \\
    s_x(g) &\coloneqq t_x(g)^{-1}g,
  \end{align}
  then \(o_x\in \Orb_G(x)\). By condition 1 in \cref{def:transversal}, there
  exists \(h\in\Stab_G(x)\) such that \(t(g) = gh\), and in particular \(s_x(g)
  = h\in \Stab_G(x)\). So there is a well-defined function \(\phi_x(g)
  \coloneqq (o_x(g), s_x(g))\) with \(\phi_x : G \rightarrow
  \Orb_G(x)\times\Stab_G(x)\).

  To see that \(\phi_x\) is injective, suppose that \(\phi_x(f) = \phi_x(g)\).
  Then \(o_x(f) = o_x(g)\), so \(f\cdot x = g\cdot x\), and therefore \(f\in
  g\Stab_G(x)\). Condition 2 in \cref{def:transversal} implies that
  \(t_x(f)=t_x(g)\), and since \(s_x(g) = s_x(f)\) we have \(t_x(f)^{-1}f =
  t_x(g)^{-1}g\), so \(f = g\).

  The orbit-stabilizer theorem implies that \(\abs{G} =
  \abs{\Orb_G(x)}\abs{\Stab_G(x)}\), and therefore if \(\phi_x\) is injective
  it must also be bijective.
\end{proof}

\section{Codecs for ordered objects}\label{app:ordered-codecs}

Codecs for strings and graphs can be composed from the primitive codecs introduced in \cref{sec:codecs}:

\hfill\begin{minipage}[t]{0.48\textwidth}
\begin{pycodesmall}
def String(ps, length):
  def encode(m, string):
    assert len(string) == length
    for c in reversed(string):
      m = Categorical(ps).encode(m, c)
    return m

  def decode(m):
    string = []
    for _ in range(length):
      m, c = Categorical(ps).decode(m)
      string.append(c)
    return m, str(string)
  return Codec(encode, decode)
\end{pycodesmall}
\end{minipage}\vline\hfill%
\begin{minipage}[t]{0.45\textwidth}
\begin{pycodesmall}
def ErdosRenyi(n, p):
  def encode(m, g):
    assert len(g) == n
    for i in reversed(range(n)):
      for j in reversed(range(i)):
        e = g[i][j]
        m = Bernoulli(p).encode(m, e)
    return m

  def decode(m):
    g = []
    for i in range(n):
      inner = []
      for j in range(i)
        m, e = Bernoulli(p).decode(m)
        inner.append(e)
      g.append(inner)
    return (m, g)
  return Codec(encode, decode)
\end{pycodesmall}
\end{minipage}

Left: Codec for fixed-length strings implemented by applying a
\ttt{Categorical} codec to each character. Right: Codec for graphs respecting
an Erdős-Rényi distribution \(G(n, p)\), implemented by applying the
\ttt{Bernoulli} codec to each edge.

\section{
A uniform codec for cosets of a permutation group
}\label{app:perm-codecs}
Shuffle coding, as described in \Cref{sec:method}, requires that we can
encode and decode left cosets in \(\Sn\) of the automorphism group of a
permutable object. In this appendix we describe a codec for cosets of an
\emph{arbitrary} permutation group characterized by a list of generators. We
first describe the codec, which we call \ttt{UniformLCoset}, on a high level
and then in \Cref{sec:Sncodec,sec:Autcodec}, we describe the two main
components in more detail.

The optimal rate for a uniform coset codec is equal to the log of the number of
cosets, that is
\begin{equation}
  \log\frac{\abs{\Sn}}{\abs{H}} = \log n! - \log\abs{H}.
\end{equation}
This rate expression hints at an encoding method: to encode a coset, we first
decode a choice of element of the coset (equivalent to decoding a choice of
element of \(H\) and then multiplying it by a canonical element of the coset),
and then encode that chosen element using a uniform codec on \(\Sn\).  Note
that if the number of cosets is small we could simply encode the index of the
coset directly, but in practice this is rarely feasible.

The following is a concrete implementation of a left coset codec:

\begin{pymathnumberedcode}
def UniformLCoset(grp):                         |\textrm{Effects on} $l(m)$\textrm{:}|
  def encode(m, s):
    s_canon = coset_canon(grp, s)
    m, t = UniformPermGrp(grp).decode(m)        |$-\log\abs{H}$|
    u = s_canon * t
    m = UniformS(n).encode(m, u)                |$+\log(n!)$|
    return m

  def decode(m):
    m, u = UniformS(n).decode(m)
    s_canon = coset_canon(subgrp, u)
    t = inv(s_canon) * u
    m = UniformPermGrp(grp).encode(m, t)
    return m, s_canon
  return Codec(encode, decode)
\end{pymathnumberedcode}

The codecs \ttt{UniformS} and \ttt{UniformPermGrp} are described in
\Cref{sec:Sncodec} and \Cref{sec:Autcodec} respectively. \ttt{UniformS(n)} is a
uniform codec over the symmetric group \(\Sn\), and \ttt{UniformPermGrp} is a
uniform codec over elements of a given permutation group, i.e., a subgroup of
\(\Sn\).

We use a \emph{stabilizer chain}, discussed in more detail in
\Cref{sec:Autcodec}, which is a computationally convenient
representation of a permutation group. A stabilizer chain allows computation of
a transversal which can be used to canonize coset elements (line 3 and line 11
in the code above). Routines for constructing and working with stabilizer
chains are standard in computational group theory, and are implemented in SymPy
(\url{https://www.sympy.org/}), as well as the GAP system
(\url{https://www.gap-system.org/}), see \citet[Chapter 4]{holt2005} for theory
and description of the algorithms. The method we use for \ttt{coset\_canon} is
implemented in the function \ttt{MinimalElementCosetStabChain} in the GAP
system.

\subsection{
  A uniform codec for permutations in the symmetric group
}\label{sec:Sncodec}
We use a method for encoding and decoding permutations based on the
Fisher-Yates shuffle \citep[][139--140]{knuth1981}. The following is a Python
implementation:
\begin{pymathnumberedcode}
def UniformS(n):
  def swap(s, i, j):
    si_old = s[i]
    s[i] = s[j]
    s[j] = si_old

  def encode(m, s):
    p = list(range(n))
    p_inv = list(range(n))
    to_encode = []
    for j in reversed(range(2, n + 1)):
      i = p_inv[s[j - 1]]
      swap(p_inv, p[j - 1], s[j - 1])
      swap(p, i, j - 1)
      to_encode.append(i)

    for j, i in zip(range(2, n + 1), reversed(to_encode)):
      m = Uniform(j).encode(m)
    return m

  def decode(m):
    s = list(range(n))
    for j in reversed(range(2, n + 1)):
      m, i = Uniform(j).decode(m)
      swap(s, i, j - 1)
    return m, s
  return Codec(encode, decode)
\end{pymathnumberedcode}
The decoder closely resembles the usual Fisher-Yates sampling method, and the
encoder has been carefully implemented to exactly invert this process. Both
encoder and decoder have time complexity in \(O(n)\).
\subsection{
  A uniform codec for permutations in an arbitrary permutation group
}\label{sec:Autcodec}
For coding permutations from an arbitrary permutation group, we use the
following construction, which is a standard tool in computational group theory
(see \citet{seress2003,holt2005}):

\begin{definition}[Base, stabilizer chain]
  Let \(H\le\Sn\) be a permutation group, and \(B=(b_0,\ldots,b_{K-1})\) a list
  of elements of \([n]\). Let \(H_0 \coloneqq H\), and \(H_k \coloneqq
  \Stab_{H_{k-1}}(b_{k-1})\) for \(k=1,\ldots,K\). If \(H_K\) is the trivial
  group containing only the identity, then we say that \(B\) is a \emph{base}
  for \(H\), and the sequence of groups \(H_0,\ldots,H_K\) is a
  \emph{stabilizer chain} of \(H\) relative to \(B\).
\end{definition}
Bases and stabilizer chains are guaranteed to exist for all permutation groups,
and can be efficiently computed using the Schreier-Sims algorithm
\citep{sims1970}. The algorithm also produces a transversal for the left cosets
of each \(H_{k+1}\) in \(H_k\) for each \(k=0,\ldots,K-1\), in a form known as
a Schreier tree \citep{holt2005}.

If we define \(O_k\coloneqq\Orb_{H_k}(b_k)\), for \(k = 0,\ldots,K-1\), then
by applying the orbit-stabilizer theorem recursively, we have \(\abs{H} =
\prod_{k=0}^{K-1}\abs{O_k}\), which gives us a decomposition  of the optimal rate that a
uniform codec on \(H\) should achieve:
\begin{equation}\label{eq:Autrate}
  \log\abs{H} = \sum_{k=0}^{K-1}\log\abs{O_k}.
\end{equation}
Furthermore, by applying \cref{lemma:OS-bijection} recursively, using the
transversals produced by Schreier-Sims, we can construct an explicit bijection
between \(H\) and the Cartesian product \(\prod_{k=0}^{K-1}O_k\).
We use this bijection, along with a sequence of uniform codecs on
\(O_0,\ldots,O_{K-1}\) for coding automorphisms at the optimal rate in
\cref{eq:Autrate}. For further details refer to the implementation.


\section{Pólya urn model details} \label{app:pu}

We implemented Pólya urn models mostly as described in \citet{severo2023rec}, with few modifications. Differently to the original implementation, we apply shuffle coding to the list of edges, resulting in a codec for the set of edges.

We also disallow edge redraws and self-loops, leading to an improved rate, as shown in \cref{sec:ablations}. This change breaks edge-exchangeability, leading to a `stochastic' codec, meaning that the code length depends on the initial message. Shuffle coding is compatible with such models. In this more general setting, the ordered log-likelihood term in the optimal rate (eq. \ref{eq:rate}) is replaced with a variational `evidence lower bound' (ELBO). The discount term is unaffected. The derivations in the main text are based on the special case of exchangeable models, where log-likelihoods are exact, for simplicity. They can be generalized with little effort and new insight.

\section{Parameter coding details}\label{sec:param-coding}

All bit rates reported for our experiments include model parameters. Once per dataset, we code the following lists of natural numbers by coding both the list length and the bit count $\left\lceil{\log m}\right\rceil$ of the maximum element $m$ with a 46-bit and 5-bit uniform codec respectively, as well as each element of the list with a  codec respecting a log-uniform distribution in [0, $\left\lceil{\log m}\right\rceil$]:

\begin{itemize}
    \item 
    A list resulting from sorting the graphs' numbers of vertices, and applying run-length coding, encoding run lengths and differences between consecutive numbers of vertices.
    \item For datasets with vertex attributes: a list of all vertex attribute counts within a dataset.
    \item For datasets with edge attributes: a list of all edge attribute counts within a dataset.
    \item For Erdős-Rényi models: a list consisting of the following two numbers: the total number of edges in all graphs, and the number of vertex pairs that do not share an edge.
\end{itemize}

Coding these empirical count parameters allows coding the data according to maximum likelihood categorical distributions. For Pólya urn models, we additionally code the edge count for each graph using a uniform codec over $[0, \frac{1}{2}n(n - 1)]$, exploiting the fact that the vertex count $n$ is already coded as described above. For each dataset, we use a single bit to code whether or not self-loops are present and adapt the codec accordingly.



\section{Compression speed} \label{app:speed}

We show compression and decompression speeds of our experiments in \Cref{tab:speed}. These speeds include time needed for gathering dataset statistics and parameter coding. The results show that for our implementation, only a small fraction of runtime is spent on finding automorphism groups and canonical orderings with \ttt{nauty}.

\begin{table}[ht]
  \centering
  \caption{
    Compression and decompression speeds in kilobytes per second (kB/s) of shuffle coding with the Erdős-Rényi (ER) and Pólya urn (PU) models, for all previously reported TU and SZIP datasets. We show SZIP compression speeds calculated from the runtimes reported in \citet{choi2012} for comparison. All results are based on the ordered ER rate as the reference uncompressed size. All shuffle coding speeds are for a single thread on a MacBook Pro 2018 with a 2.7GHz Intel Core i7 CPU. We also report the share of time spent on \ttt{nauty} calls that determine the canonical
    ordering and generators of the automorphism group of all graphs.
  }\label{tab:speed}
  \begin{tabular}{@{}lrrrrrr@{}}
    \toprule
                   &     &\multicolumn{2}{c}{ER}&\multicolumn{2}{c}{PU}&\multicolumn{1}{c}{SZIP}\\ 
                   \cmidrule(lr){3-4}\cmidrule(l){5-6}\cmidrule(l){7-7}
    Dataset        &nauty  &encode &decode &encode &decode& encode\\\midrule\addlinespace\tbf{TU by type}\\
    Small molecules&15\%   &54     &56     &--     &--    &--    \\
    Bioinformatics &2\%    &51     &66     &--     &--    &--    \\
    Computer vision&3\%    &25     &28     &--     &--    &--    \\
    Social networks&<1\%   &0.440  &0.467  &--     &--    &--    \\
    Synthetic      &7\%    &98     &110    &--     &--    &--    \\\midrule\addlinespace\tbf{Small molecules}\\
    MUTAG                      &7\%    &115    &122    &51     &40    &--    \\
    MUTAG (with attributes)    &16\%   &135    &141    &67     &62    &--    \\
    PTC\_MR                    &8\%    &107    &103    &50     &45    &--    \\
    PTC\_MR (with attributes)  &18\%   &117    &125    &67     &62    &--    \\
    ZINC\_full                 &7\%    &105    &105    &50     &47    &--    \\\addlinespace\tbf{Bioinformatics}\\
    PROTEINS                   &3\%    &88     &94     &30     &30    &--    \\\addlinespace\tbf{Social networks}\\
    IMDB-BINARY                &4\%    &17     &18     &8      &8     &--    \\
    IMDB-MULTI                 &3\%    &11     &12     &6      &5     &--    \\\midrule\addlinespace\tbf{SZIP}\\
    Airports (USAir97)         &1\%    &82     &78     &5      &5     &164   \\
    Protein interaction (YeastS)&8\%   &2.442  &2.391  &1.238  &0.859 &77    \\
    Collaboration (geom)       &<1\%   &0.004  &0.005  &0.005  &0.005 &64    \\
    Collaboration (Erdos)      &15\%   &0.025  &0.025  &0.024  &0.024 &18    \\
    Genetic interaction (homo) &7\%    &0.180  &0.154  &0.117  &0.141 &32    \\
    Internet (as)              &13\%   &0.002  &0.003  &0.002  &0.002 &7     \\
  \end{tabular}
\end{table}

\section{Model ablations}\label{sec:ablations}

We present results of additional ablation experiments on the PnC datasets in \Cref{tab:ablation}. We do an ablation that uses a uniform distribution for vertex and edge attributes with an Erdős-Rényi model (unif. ER). There is a clear advantage to coding maximum-likelihood categorical parameters (ER), justifying it as the approach used throughout this paper. We also show the rates obtained by the original method proposed in \citet{severo2023rec} (PU redr.), demonstrating a clear rate advantage of our approach disallowing edge redraws and self-loops (PU) in the model.

\begin{table}[ht]
  \caption{Model ablations compared to PnC. All results are in bits per edge.}\label{tab:ablation}
  \centering
  \begin{tabular}{@{}lCCCCr@{}l@{}}
    \toprule
                                &\multicolumn{4}{c}{Shuffle coding}\\ \cmidrule(lr){2-5}
    Dataset                     &unif. ER        &ER&PU       &PU redr.         &\multicolumn{2}{c}{PnC}\\\midrule\addlinespace\tbf{Small molecules}\\
    MUTAG                       &--  &\textbf{1.88}&2.66             &2.81       &2.45&\(\pm\)0.02 \\
    MUTAG (with attributes)     &6.37&\textbf{4.20}&4.97         &5.13       &4.45&            \\
    PTC\_MR                     &--  &\textbf{2.00}&2.53             &2.74       &2.97&\(\pm\)0.14 \\
    PTC\_MR (with attributes)   &8.04&\textbf{4.88}&5.40         &5.61       &6.49&\(\pm\)0.54 \\
    ZINC\_full                  &--  &\textbf{1.82}&2.63             &2.75       &1.99&                  \\\addlinespace\tbf{Bioinformatics}\\
    PROTEINS                    &--  &3.68      &\textbf{3.50}       &3.62       &3.51&\(\pm\)0.23 \\\addlinespace\tbf{Social networks}\\
    IMDB-BINARY                 &--  &2.06      &1.50       &2.36       &\textbf{0.54}&            \\
    IMDB-MULTI                  &--  &1.52      &1.14       &2.17       &\textbf{0.38}&            \\
    \bottomrule
  \end{tabular}
\end{table}

\end{document}